%% file: main.tex
\documentclass[conference]{IEEEtran}

\IEEEoverridecommandlockouts

\usepackage{cite}
\usepackage{amsmath,amssymb,amsfonts, amsthm}
\usepackage{algorithmic}
\usepackage{graphicx}
\usepackage{textcomp}
\usepackage{xcolor}
\def\BibTeX{{\rm B\kern-.05em{\sc i\kern-.025em b}\kern-.08em
    T\kern-.1667em\lower.7ex\hbox{E}\kern-.125emX}}

\usepackage{soul}
\usepackage{hyperref}
\usepackage{enumitem} 
\usepackage{nccmath} 
\usepackage{mathtools}
\usepackage[italicdiff]{physics}

\input{tikz_preamble.tex}

\theoremstyle{plain}
\newtheorem{theorem}{Theorem}
\newtheorem*{theorem*}{Theorem}
\newtheorem{lemma}{Lemma}
\newtheorem*{lemma*}{Lemma}

\newtheorem*{proposition*}{Proposition}

\theoremstyle{definition}
\newtheorem{assumption}{Assumption}
\newtheorem{definition}{Definition}

\theoremstyle{remark}

\usepackage{subfigure}

\newcommand{\E}{\mathbb{E}}
\newcommand{\Var}{\mathbb{V}}

\newcommand{\mhat}{\widehat{\mathbf{m}}_{-k,t}}

\newcommand{\mhatt}{\widehat{\mathbf{m}}_t}
\newcommand{\sigmahat}{\widehat{\sigma}}
\newcommand{\xbf}{\mathbf{X}}
\newcommand{\xbfcap}{\underline{\mathbf{X}}}
\newcommand{\thbf}{\boldsymbol{\theta}}

\newcommand{\zbf}{\mathbf{Z}}
\newcommand{\zbfcap}{\underline{\mathbf{Z}}}

\newcommand{\mbf}{\mathbf{m}}

\begin{document}

\title{Distributed Stochastic Gradient Descent with Cost-Sensitive and Strategic Agents 
}

\author{\IEEEauthorblockN{Abdullah Basar Akbay, Cihan Tepedelenlioglu}
\IEEEauthorblockA{\textit{School of Electrical, Computer and Energy Engineering} \\
\textit{Arizona State University}\\
Tempe, Arizona \\
aakbay@asu.edu, cihan@asu.edu}
}

\maketitle

\begin{abstract}
This study considers a federated learning setup where cost-sensitive and strategic agents train a learning model with a server. During each round, each agent samples a minibatch of training data and sends his gradient update. As an increasing function of his minibatch size choice, the agent incurs a cost associated with the data collection, gradient computation and communication. The agents have the freedom to choose their minibatch size and may even opt out from training. To reduce his cost, an agent may diminish his minibatch size, which may also cause an increase in the noise level of the gradient update. The server can offer rewards to compensate the agents for their costs and to incentivize their participation but she lacks the capability of validating the true minibatch sizes of the agents. To tackle this challenge, the proposed reward mechanism evaluates the quality of each agent's gradient according to the its distance to a reference which is constructed from the gradients provided by other agents. It is shown that the proposed reward mechanism has a cooperative Nash equilibrium in which the agents determine the minibatch size choices according to the requests of the server.
\end{abstract}


\input{01_Introduction.tex}
\input{02_System_Model.tex}

\input{03_ProblemFormulation}

\input{04_Reward_Mech.tex}
\input{05_Experiments.tex}
\input{06_Conclusion.tex}

\bibliographystyle{IEEEtran}
\bibliography{main}
\numberwithin{equation}{section}
\newpage
\onecolumn
\appendices
\input{App01}
\input{App02}

\end{document}

%% file: tikz_preamble.tex
\usepackage{tikz}
\usetikzlibrary{arrows,shapes,positioning,shadows,trees,mindmap,matrix,snakes}
\usepackage[edges]{forest}
\usetikzlibrary{arrows.meta}
\colorlet{linecol}{black!75}
\usepackage{xkcdcolors}

\usepackage{tikz}
\usetikzlibrary{backgrounds}
\usetikzlibrary{arrows,shapes, shapes.multipart}
\usetikzlibrary{tikzmark}
\usetikzlibrary{calc}


%% file: 01_Introduction.tex
\section{Introduction}

Federated learning (FL) has emerged as a distributed computing paradigm where multiple \textit{agents} train a machine learning model in a collaborative manner under the coordination of a \textit{server}\footnote{We refer to the server as ``she" and an agent as ``he".} without transferring or disclosing their raw data \cite{Konecny45630}. The agents compute and communicate local updates (typically stochastic gradients) to the model, and then the server performs a global update  \cite{McMahan2018learning}. The agents can be different organizations, such as medical or financial corporations, which cannot share their confidential data with the server due to legal or business related concerns. Although FL has the potential to alleviate privacy risks associated with centralized machine learning, its operation counts on the assistance of the agents who incur privacy, computation, communication and energy costs for their efforts. Therefore, compensation of their losses in an effective fashion is vital to ensure their cooperation.

In this study, we consider strategic agents, i.e., they are self-interested and rational entities and they seek to maximize their utilities. During each round, each agent locally samples a minibatch of training data points, and sends their stochastic gradients. As a function of his minibatch size choice, each agent incurs a cost and the server compensates him through a reward mechanism. Note that the server cannot directly observe or verify the minibatch size choices of the agents. Therefore, an agent may attempt to reduce his cost by collecting less data points, while still claiming the same reward from the server. This can increase the noise levels of the stochastic gradients and can severely hamper the training efforts.

The described framework diverges from a classical FL setup where the participants are assumed to be submissive clients, who always carry out their dictated tasks according to a pre-established protocol. To address this challenge, we design a reward mechanism which constructs a reference gradient based on the stochastic gradients collected from every participating agent and then rewards each agent based on the distance between each agent's gradient update and the constructed reference gradient. We show that the proposed reward mechanism has a cooperative Nash equilibrium where the agents follow the lead of the server. This approach enables a feasible FL training in the presence of strategic agents by ensuring the reliability of local model updates.


\subsection{Related Work}

There is a growing literature that focuses on game theoretic approaches and incentives for FL \cite{FengNiyato19, PandeyTran20, KangXiong19, DingFang21, Tang21, DongZhang20, DengLyu21, Tahanian21, FLEconSurvey22, Akbay22}. The authors in \cite{FengNiyato19} consider a Stackelberg game model where the agents determine the price per unit of their data, and the server determines the size of training data to be acquired from each client. In \cite{PandeyTran20}, a crowdsourcing framework is proposed to encourage the agents that improves their local accuracy. Contract based reward mechanisms are also studied for encouraging high-quality agent's participation in FL \cite{KangXiong19, DingFang21}. The authors in \cite{Tang21} study a similar framework from a public goods perspective where the agents are also interested in the accuracy of the trained model. Dong and Zhang \cite{DongZhang20} introduced a market oriented framework and formulated a hierarchical Stackelberg game. The authors in \cite{DengLyu21} propose a quality-aware incentive mechanism which is based on the estimates of the learning quality. A game-theoretical approach is adopted in \cite{Tahanian21} to classify the malicious agents by modeling the gradient aggregation process as a mixed-strategy game. A recent survey on the economic and game theoretic approaches in FL can be found in \cite{FLEconSurvey22}.

The work closest in spirit to ours is \cite{Akbay22}, where the authors model the interactions between the server and the agents as repeated games under the assumption that the server is not capable of directly authenticating the binary actions of the agents (cooperative or defective). While their setting is similar to ours, we impose a more general problem where the actions of the agents belong to a discrete set rather than a binary set. Furthermore, their approach is based on a zero-determinant strategy \cite{PressDyson12} which models the server as another selfish player, whereas we follow an incentive mechanism design approach.

\subsection{Notational Conventions}

All vectors are assumed to be column vectors. We use boldface type (e.g., $\mathbf{x}, \mathbf{X}$) to denote vectors. We use boldface type with an underscore to indicate a matrix ($\underline{\mathbf{x}}, \underline{\mathbf{X}}$). Random variables are upper case (e.g., $X, \mathbf{X}, \underline{\mathbf{X}}$). The transpose of a vector $\mathbf{x}$ is denoted by $\mathbf{x}^\top$. We denote the zero vector with $\boldsymbol{0}$. Given a vector $\mathbf{x}$, we denote by $\| \mathbf{x} \|$ its Euclidean norm, i.e., $\| \mathbf{x} \| = \sqrt{\mathbf{x}^\top \mathbf{x}}$. The variance of a random vector $\mathbf{X}$ is defined as $\Var[\mathbf{X}] := \E \left[ \| \mathbf{X} - \E[\mathbf{X}] \|^2 \right] = \E \left[\| \mathbf{X} \|^2 \right] - \left \| \E [\mathbf{X}] \right\|^2$. Given a set $\mathcal{S}$, we denote by $|\mathcal{S}|$ its cardinality.

%% file: 02_System_Model.tex
\section{System Model}

This study considers a distributed implementation of stochastic gradient descent (SGD) algorithm \cite{Dean12} with a \textit{server} and a set of strategic \textit{agents} $\mathcal{K} = \{1, 2, \dots, K\}$. Let $\zbf \in \mathcal{Z}$ denote the data generated from some unknown underlying distribution $\mathcal{D}$. Consider a loss function $\ell: \mathcal{Z} \times \mathbb{R}^{d} \to \mathbb{R}_{\geq 0}$, where $\ell (\mathbf{z}, \thbf)$ measures the loss associated with a realization $\mathbf{z}$ of the data under the model parameter choice $\thbf$. The \textit{population loss} function, $L(\thbf)$, is defined by $L(\thbf) : = \E[\ell (\thbf;\zbf)] $ where the expectation is with respect to $\mathcal{D}$. The goal of the \textit{server} is to solve $\min_{\thbf} L(\thbf)$. We make the following assumptions on the cost function $\ell$ and population loss function $L$.
\begin{assumption}
$\ell(\mathbf{z}, \thbf)$ is continuously differentiable in $\thbf$ for any $\mathbf{z} \in \mathcal{Z}$ and $\nabla L (\thbf) = \nabla \E [\ell (\thbf; \zbf)] = \E \left[ \nabla_{\thbf} \ell (\thbf; \zbf) \right]$ holds.
\end{assumption} 

As illustrated in Figure~\ref{Fig:ParameterServerArchitecture}, at the beginning of each round $t$, the server broadcasts the latest model iterate $\thbf_{t-1}$ to the agents who then decide how many data points they will collect in order to produce a stochastic gradient for round $t$. The agent $k$ collects a minibatch of size $S_{k,t} \in \mathcal{S} := \{0, 1, \dots, n\}$ which consists of independent and identically distributed (i.i.d.) data points $\zbfcap_{k,t} := [\zbf_{k,t,1}\ \dots \ \zbf_{k,t,S_{k,t}}]$ drawn from $\mathcal{D}$, computes the gradient of the local empirical loss
\begin{equation} \label{Eq:XDef}
\xbf_{k,t} := \frac{1}{S_{k,t}} \sum\nolimits_{i=1}^{S_{k,t}} \nabla_{\thbf} \ \ell \big( \thbf_{t-1}; \zbf_{k,t,i} \big), \quad S_{k,t} > 0,
\end{equation}
and sends $\xbf_{k,t}$ to the server. The agent does not exchange his local data $\zbfcap_{k,t}$ with other users or the server. The agent can decline to participate in any given round ($S_{k,t} = 0$) and $\xbf_{k,t} = \bot$ indicates ``nonparticipation". 

\begin{figure}
\input{Figure_01.tex}
 \vspace{-4mm}
\caption{At the beginning each round $t$, the server broadcasts the current iterate $\thbf_{t-1}$. Each agent locally samples a minibatch of training data points $\zbfcap_{k,t}$ and computes his stochastic gradient $\xbf_{k,t}$. The server aggregates the gradients to update the global training model as $\thbf_{t}$.}
\label{Fig:ParameterServerArchitecture}
\end{figure}
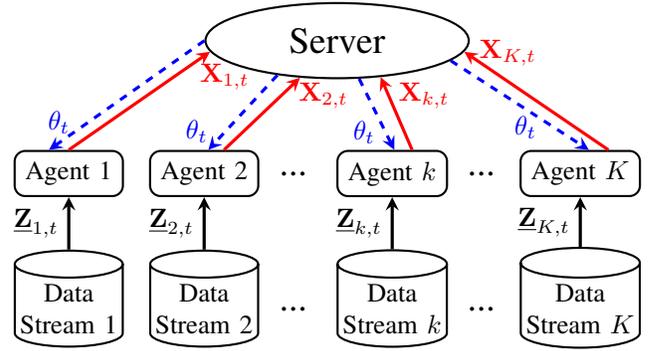

As a function of the minibatch size choice $S_{k,t}$, agent $k$ incurs a \textit{cost} captured by $h: \mathbb{R}_{\geq 0} \to \mathbb{R}_{\geq 0}$. In addition to the time and power expenses due to computation and transmission of the stochastic gradient $\xbf_{k,t}$ \eqref{Eq:XDef} and its communication to the server, $h(S_{k,t})$ may also incorporate the effort of data acquisition and associated privacy risks for the agent. Formally, we make the following assumptions on $h$.

\begin{assumption} \label{Assumption:CostFunction}
The cost function $h(\cdot)$ is strictly increasing with $h(0) = 0$, and it is the same for every agent. Further, $h(\cdot)$ is twice differentiable and convex on $\mathbb{R}_{>0}$.
\end{assumption}

The gradient of the population loss, $\nabla L(\thbf)$, is referred to as the \textit{population gradient}. Let $\mbf_t$ denote the population gradient evaluated at the current iterate
\begin{equation} 
    \mbf_t :=  \nabla L(\thbf_{t-1}) = \E [\nabla_{\thbf} \ell (\thbf_{t-1}; \zbf )]. \label{Eq:mtdef}
\end{equation}
The stochastic gradients can be regarded as unbiased noisy estimates of $\mbf_t$, i.e., $\E \left[ \xbf_{k,t}  |  \thbf_{t-1}, S_{k,t} \right] = \mbf_t$. Let $\mathcal{K}_t$ denote the set of participating agents, i.e., $\mathcal{K}_t := \{k \in \mathcal{K} : \xbf_{k,t} \neq \bot \}$. When $\mathcal{K}_t \neq \emptyset$, the server forms an unbiased estimate of the population gradient, $\mhatt$, and updates the model iterate as
\begin{align} \label{Eq:SGDBasic}
	\thbf_{t} = \thbf_{t-1} - \eta_t \mhatt(\xbfcap_t) \text{ and } \mhatt (\xbfcap_t) = \frac{1}{|\mathcal{K}_t|} \sum_{k \in \mathcal{K}_t} \xbf_{k,t},
\end{align}
where $\eta_t \! \geq \! 0$ is a stepsize and $\xbfcap_t =\! [\xbf_{1,t} \dots \xbf_{K,t}]$ is a matrix that collects the gradients sent by the agents in round $t$. If $\mathcal{K}_t  = \emptyset$, the server does not update the parameter, $\thbf_t = \thbf_{t-1}$.

%% file: Figure_01.tex
\centering
\begin{tikzpicture}
	\tikzstyle{server}=[ellipse, 
						draw, 
						thick,
						minimum width = 35mm,
						minimum height = 10mm,
						inner sep=2pt]

	\tikzstyle{agent}=[rectangle, 
						rounded corners,
						draw, 
						thick, 
						minimum width = 5mm,
						minimum height = 6mm,
						inner sep=4pt]
						
	\tikzstyle{dataStream}=[cylinder,
							shape border rotate=90,
							aspect=0.25,
							draw,
							thick,
							inner sep = 2pt]
	
	\node[agent] (Agent1) {Agent 1};				
	\node[agent, right of=Agent1, node distance=18mm] (Agent2) {Agent 2};
	\node[text width=3cm, right of=Agent2, node distance=25mm] {\Large ...};
	\node[agent, right of=Agent2, node distance=25mm] (Agentk) {Agent $k$};
	\node[text width=3cm, right of=Agentk, node distance=25mm]  {\Large ...};
	\node[server, above right of=Agent2, node distance=2.5cm] (server) {\Large Server};
	\node[agent, right of=Agentk, node distance=25mm] (AgentLast) {Agent $K$};
	\node[dataStream, below of=Agent1, align=center, node distance = 18mm] (DataStream1) {Data\\ Stream 1};
	\node[dataStream, below of=Agent2, align=center, node distance = 18mm] (DataStream2) {Data\\ Stream 2};
	\node[text width=3cm, right of=DataStream2, node distance=25mm] {\Large ...};
	\node[dataStream, below of=Agentk, align=center, node distance = 18mm] (DataStreamk) {Data\\  Stream $k$};
	\node[text width=3cm, right of=DataStreamk, node distance=25mm] {\Large ...};
	\node[dataStream, below of=AgentLast, align=center, node distance = 18mm] (DataStreamLast) {Data\\ Stream $K$};

	\draw[dashed, -{stealth}, blue, very thick] (server.180) to node[near end, left]{$\theta_t \hspace{1mm}$} (Agent1.130);	
	\draw[-{stealth}, red, very thick] (Agent1.90) to node[near end, right]{$\hspace{2mm} \xbf_{1,t}$} (server.185);
	
	\draw[dashed, -{stealth}, blue, very thick] (server.210) to node[near end, left]{$\theta_t \hspace{1mm}$} (Agent2.80);	
	\draw[-{stealth}, red, very thick] (Agent2.50) to node[near end, right]{$\hspace{1mm} \xbf_{2,t}$} (server.225);
	
	\draw[dashed, -{stealth}, blue, very thick] (server.300) to node[near end, left]{$\theta_t$} (Agentk.85);	
	\draw[-{stealth}, red, very thick] (Agentk.50) to node[near end, right]{$\xbf_{k,t}$} (server.320);

	\draw[dashed, -{stealth}, blue, very thick] (server.350) to node[near end, left]{$\theta_t \hspace{1mm}$} (AgentLast.70);	
	\draw[-{stealth}, red, very thick] (AgentLast.40) to node[near end, above]{$\hspace{2mm} \xbf_{K,t}$} (server.355);
	
	\draw[-{stealth}, very thick] (DataStream1.90) to node[left]{$ \zbfcap_{1,t}$}(Agent1.270);
	\draw[-{stealth}, very thick] (DataStream2.90) to node[left]{$ \zbfcap_{2,t}$}(Agent2.270);
	\draw[-{stealth}, very thick] (DataStreamk.90) to node[left]{$ \zbfcap_{k,t}$}(Agentk.270);
	\draw[-{stealth}, very thick] (DataStreamLast.90) to node[left]{$ \zbfcap_{K,t}$}(AgentLast.270);

\end{tikzpicture}

%% file: 03_ProblemFormulation.tex
\section{Problem Formulation}

\begin{figure}[!t]
\centering
\includegraphics[width=3.4in]{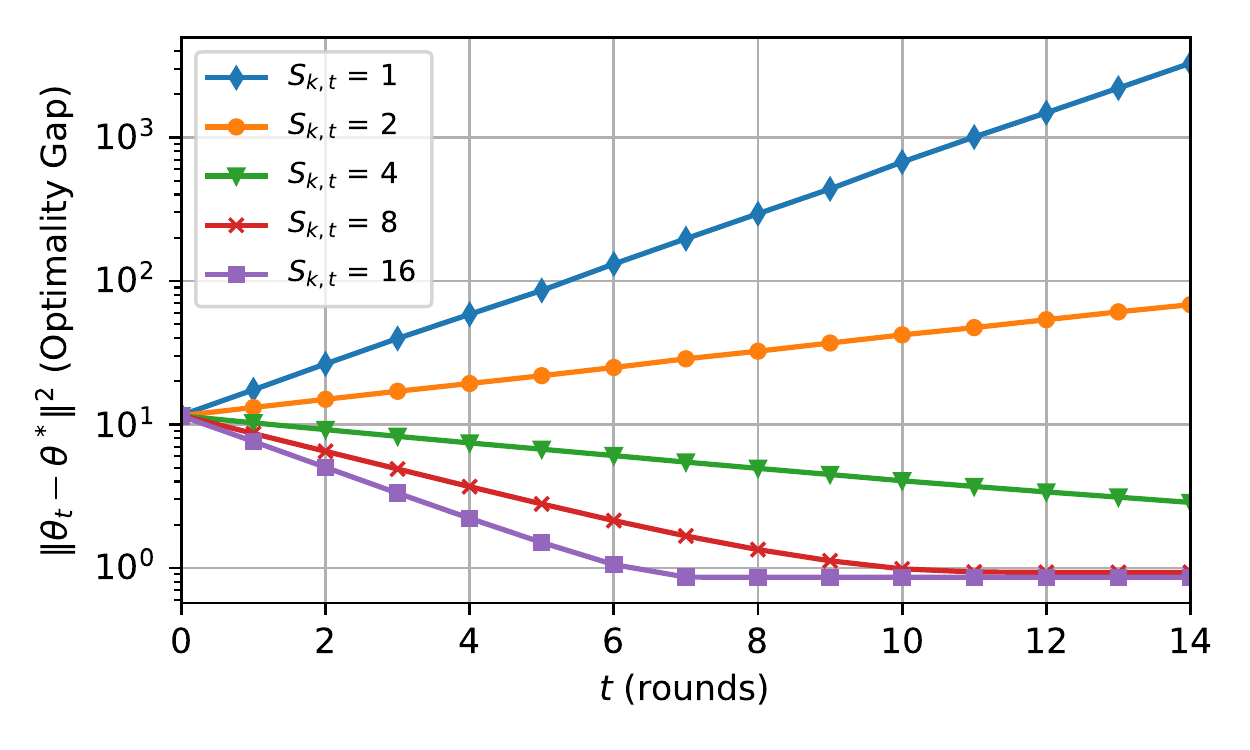}
\vspace{-4mm}
\caption{Linear regression experiments to demonstrate the impact of strategic minibatch size choices of agents on the convergence performance of SGD. }
\label{Fig:ProblemFormulation}
\end{figure}

In a conventional SGD based training, an increase in the mini-batch size allows an increase in the stepsize due to the reduced variance of the stochastic gradients and yields faster convergence \cite{Bottou18}. However, in this study, the strategic agents have the freedom to choose their minibatch sizes, $S_{k,t}$'s, and their participation is voluntary. Consequently, the variance of the stochastic gradients, denoted by $\Var [\xbf_{k,t}]$, also depend on the minibatch size choices of the agents:
\begin{align*}
\Var [\xbf_{k,t} | \thbf_{t-1}, S_{k,t}] \!=\! \Var \left[ \frac{1}{S_{k,t}} \sum_{i=1}^{S_{k,t}} \nabla_{\thbf} \ell \big( \thbf_{t-1}; \zbf_{k,t,i} \big) \right] \!=\! \frac{\sigma_t^2}{S_{k,t}}
\end{align*}
where $\sigma_t^2$ is the variance of a single sample stochastic gradient, i.e., $\sigma_t^2 = \E \left[ \| \nabla_{\thbf} \ell (\thbf_{t-1}; \zbf) - \mbf_{t}\|^2 \right]$. This observation implies that the variance of the gradient estimator $\mhatt$ can vary based on the minibatch size choices of the strategic agents.

The lack of direct control over the minibatch size choices presents a serious challenge for the server in the selection of the stepsizes throughout the training. In Figure~\ref{Fig:ProblemFormulation}, we present the results for a linear regression problem to demonstrate the possible adverse consequences of this issue. The setup of these experiments is detailed in Section~\ref{Section:Experiments}. In this experiment, the stepsize $\eta_t$ is selected assuming that the agents select $S_{k,t} = 8$ in every iteration. Selection of minibatch sizes smaller than $8$ can hinder convergence or even cause the algorithm to diverge. On the other hand, we do not really observe a much faster convergence when the agents select $S_{k,t}$'s larger than $8$ because this information is not available to the server and the stepsize is not adjusted to enjoy the benefit of variance reduction.

To overcome the described challenge, we will follow a game-theoretic approach where the agents are incentivized to follow a server-requested minibatch size sequence $\nu_t$ where they all set $S_{k,t} \!=\! \nu_t$. Let $\boldsymbol{\nu}_t \! := \! [\nu_t \dots \nu_t]_{K \times 1}$. The server employs a reward mechanism, $R_{k,t} (\xbfcap_t) : \mathcal{X}^K \to \mathbb{R}$, where $\mathcal{X} := \mathbb{R}^d \cup \{\bot\}$, and $R_{k,t}$ specifies the amount of reward for agent $k$ in round $t$. Note that $R_{k,t}$ is a function of the gradients sent by the agents in round $t$. Let $\mathbf{R}_t  = [R_{1,t}\ \dots \ R_{K,t}]$. It is assumed that nonparticipating agents do not receive any reward, i.e., $R_{k,t} = 0$ if $\xbf_{k,t} = \bot$. Since the agents are strategic, each one of them aims to maximize his individual expected utility which can be calculated as the difference between the expected reward and cost:
\begin{align} \label{Eq:UtilityAgent}
	u_{k,t} (\mathbf{S}_t) := \E \left[ R_{k,t} \ |\ \thbf_{t-1}, \mathbf{S}_t \right] - h(S_{k,t}).
\end{align}
The expectation is with respect to the data distribution $\mathcal{D}$, and it is conditioned on the current iterate $\thbf_{t-1}$ and the minibatch size vector.

We call an agent's minibatch size choice $S_{k,t}$ \textit{cooperative} if $S_{k,t} = \nu_t$. Each agent is free to choose his minibatch size $S_{k,t}$ and will only make a cooperative choice if it maximizes their expected utility \eqref{Eq:UtilityAgent}. Due to the form of the server's reward mechanism, the optimal choice for agent $k$ can only be described given the choices of agents other than $k$. Thus, we will adopt \textit{Nash equilibrium} as our main game-theoretic solution concept.

\begin{definition} 
Denote the minibatch sizes of all agents but $k$ by $\mathbf{S}_{-k,t} = [S_{1,t}, \dots, S_{k-1,t}, S_{k+1,t}, \dots, S_{K,t}]$. Agent $k$'s \textit{best response} to the minibatch size vector $\mathbf{S}_{-k,t}$ is $s_{k,t}^* \in \mathcal{S}$ such that
\begin{equation*}
u_{k,t} (S_{k,t} = s_{k,t}^*, \mathbf{S}_{-k,t}) \geq u_{k,t} (S_{k,t} = s_{k,t}, \mathbf{S}_{-k,t})
\end{equation*}
for all $s_{k,t} \in \mathcal{S}$. A minibatch size vector $\mathbf{S}_{t}$ is a \textit{Nash equilibrium} if, for all agents $k$, $S_{k,t}$ is a best response to $\mathbf{S}_{-k,t}$. We say that a reward mechanism $\mathbf{R}_t$ has a \textit{\textbf{cooperative equilibrium}} if $\mathbf{S}_t = \boldsymbol{\nu}_t$ is a Nash equilibrium.
\end{definition}

Our goal is to incentivize agents with the reward mechanism to make cooperative choices throughout the learning process. In a cooperative equilibrium, this objective is achieved as no agent can benefit from unilaterally changing his cooperative choice, given that other agents are making cooperative choices. We would also like the reward mechanism to have high budgetary efficiency. We say that a reward mechanism $\mathbf{R}_t$ is \textit{\textbf{budget balanced}} if no agent is rewarded more than his cost in a cooperative equilibrium:
\begin{align} \label{Eq:IRDef}
    \E \left[R_{k,t} \ |\ \thbf_{t-1}, \mathbf{S}_t = \boldsymbol{\nu}_t \right] = h(\nu_t), 
\end{align}
Recall that the reward, utility and cost of a nonparticipating agent is equal to 0. Since the agents can always choose not to participate $S_{k,t} = 0$, the expected reward of an agent cannot be less than his cost $h(\nu_t)$ in a cooperative equilibrium.

%% file: 04_Reward_Mech.tex
\section{Reward Mechanism Design}

We propose an output agreement\footnote{\textit{Output agreement} is a term first introduced in \cite{vonAhn04}, for an image labeling game. It captures the notion of rewarding an agent only if his answer is the same as that of another randomly picked agent.} based reward mechanism where the reward of an agent is determined based on the distance between the gradient of the agent and the sample mean vector of the gradients that are sent by other agents. In particular, the proposed reward mechanism, for $\xbf_{k,t} \neq \bot$ and $\mathcal{K}_{-k,t} \neq \emptyset$, has the following form:
\begin{align} \label{Eq:RewardMechanismKnownSigma}
R_{k,t} \!=\! h(\nu_t) \!+\! h'(\nu_t) \left( \nu_t  \frac{K}{K\!-\!1} \!-\! \frac{\nu_t^2}{\sigma_t^2} \| \xbf_{k,t} \!-\! \mhat \|^2 \right),
\end{align}
where $\mhat$ is formally defined as
\begin{align}
\mhat := \frac{1}{|\mathcal{K}_{-k,t}|} \sum\nolimits_{i \in \mathcal{K}_{-k,t}} \xbf_{i, t}.
\end{align}
If an agent does not participate in distributed learning during round $t$ ($X_{k,t} = \bot$ and $S_{k,t} = 0$), then his reward is equal to $0$. An agent also receives zero reward if he is the only participant ($\mathcal{K}_{-k,t} = \emptyset$).

Ideally, the server would want to determine the reward of agent $k$ according to the distance between his gradient $X_{k,t}$ with the population gradient $\mbf_t$ \eqref{Eq:mtdef} since the expected value of this distance would be inversely proportional to the minibatch size choice of the agent:
\begin{align} \label{Eq:GenieAidedMech}
\E \left [ \| \xbf_{k,t} \!-\! \mbf_t \|^2 | \thbf_{t-\!1}, S_{k,t} \right] \!=\! \Var [\xbf_{k,t} | \thbf_{t-1}, S_{k,t}] \!=\! \frac{\sigma_t^2}{S_{k,t}}. \!
\end{align}
In this scenario, the expected reward of agent $k$ would increase with his minibatch size choice $S_{k,t}$ and this would enable the server to incentivize cooperative strategies. Nevertheless, the population gradient $\mbf_t$ is unknown to the server, and she instead forms an estimate of it using the gradients of other agents. The following lemma finds the expected value of the distance between $\xbf_{k,t}$ and this estimate $\mhat$.
\begin{lemma} \label{Lemma:OAMech}
Conditioned on a minibatch size vector $\mathbf{S}_t$ such that $\mathbf{S}_{-k,t} \neq \boldsymbol{0}$ and $S_{k,t} \neq 0$, it follows that
\begin{align}
\E \big[ \| \xbf_{k,t} - &  \mhat \|^2 \ | \ \thbf_{t-1}, \mathbf{S}_t \big] = \nonumber \\
& = \frac{\sigma_t^2}{S_{k,t}} + \frac{1}{|\mathcal{K}_{-k,t}|^2} \sum\nolimits_{i\in \mathcal{K}_{-k,t}} \frac{\sigma_t^2}{S_{i,t}}.
\end{align}
\end{lemma}
The proof of this Lemma is relegated to appendix\footnote{For Appendices, we refer the interested reader to the full version of our paper (available on the authors' webpages).}. In contrast to \eqref{Eq:GenieAidedMech}, the expected value of $\| \xbf_{k,t} - \mhat \|^2$ is dependent on the minibatch size choices of all participating agents. The following result demonstrates that the best minibatch size choice for agent $k$ is to be cooperative, $S_{k,t} = \nu_t$, given that other agents are also cooperative. 

\begin{theorem} \label{Thm}
The output agreement based reward mechanism \eqref{Eq:RewardMechanismKnownSigma} has a cooperative Nash equilibrium and it is budget balanced.
\end{theorem}
The proof of this Theorem is relegated to appendix. According to Theorem~\ref{Thm}, the proposed reward mechanism can incentivize the agents to follow the requested minibatch size sequence $\nu_t$ without rewarding the agents more than their cost, despite the fact that the server can never verify the minibatch size choices of the agents.

The proposed output agreement mechanism still requires the knowledge of $\sigma_t^2 = \E \left[ \| \nabla_{\thbf} \ell (\thbf_{t-1}; \zbf) - \mbf_{t}\|^2 \right]$, the variance of a single sample stochastic gradient. In a real-life machine learning setting, this information may not be available to the server. Instead, we can take the same approach as in the derivation of the proposed mechanism and construct an estimate of $\widehat{\sigma}_{-k,t}^2$ from the stochastic gradients collected from agents:
\begin{align} \label{Eq:VarEstDef}
\sigmahat_{-k,t}^2 = \frac{\nu_t}{|\mathcal{K}_{-k,t}|-1} \sum\nolimits_{i \in \mathcal{K}_{-k,t}} \| \xbf_{i,t} - \mhat \|^2.
\end{align}
With this modification, The theoretical analysis of the reward mechanism becomes significantly more complicated and it is deferred to the expanded journal version of our study. Nonetheless, in the next section, we will provide experimental evaluation of the proposed reward mechanism with this modification. 

%% file: 05_Experiments.tex
\begin{figure}[!t]
\centering
\includegraphics[width=3.4in]{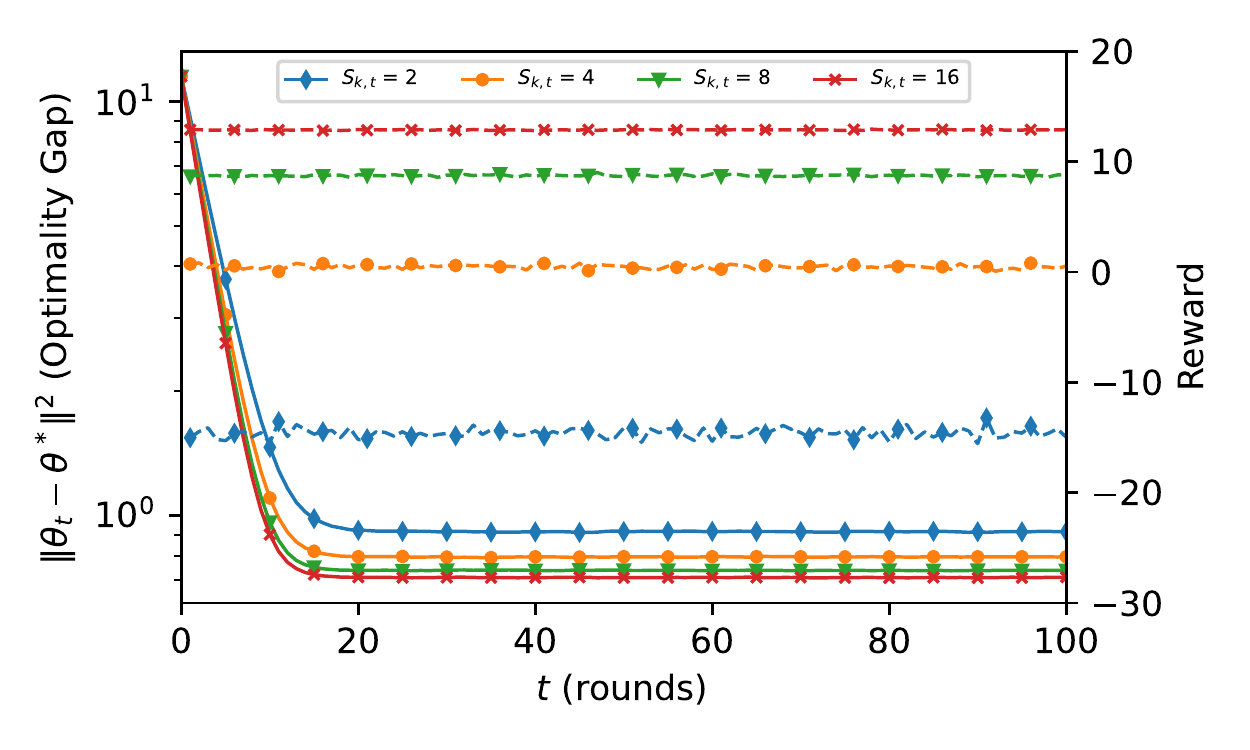}
\vspace{-4mm}
\caption{Linear regression experiments to demonstrate the performance of the proposed reward mechanism for $\nu_t = 8$.}
\label{Fig:03}
\end{figure}

\begin{figure}[!t]
\centering
\includegraphics[width=3.4in]{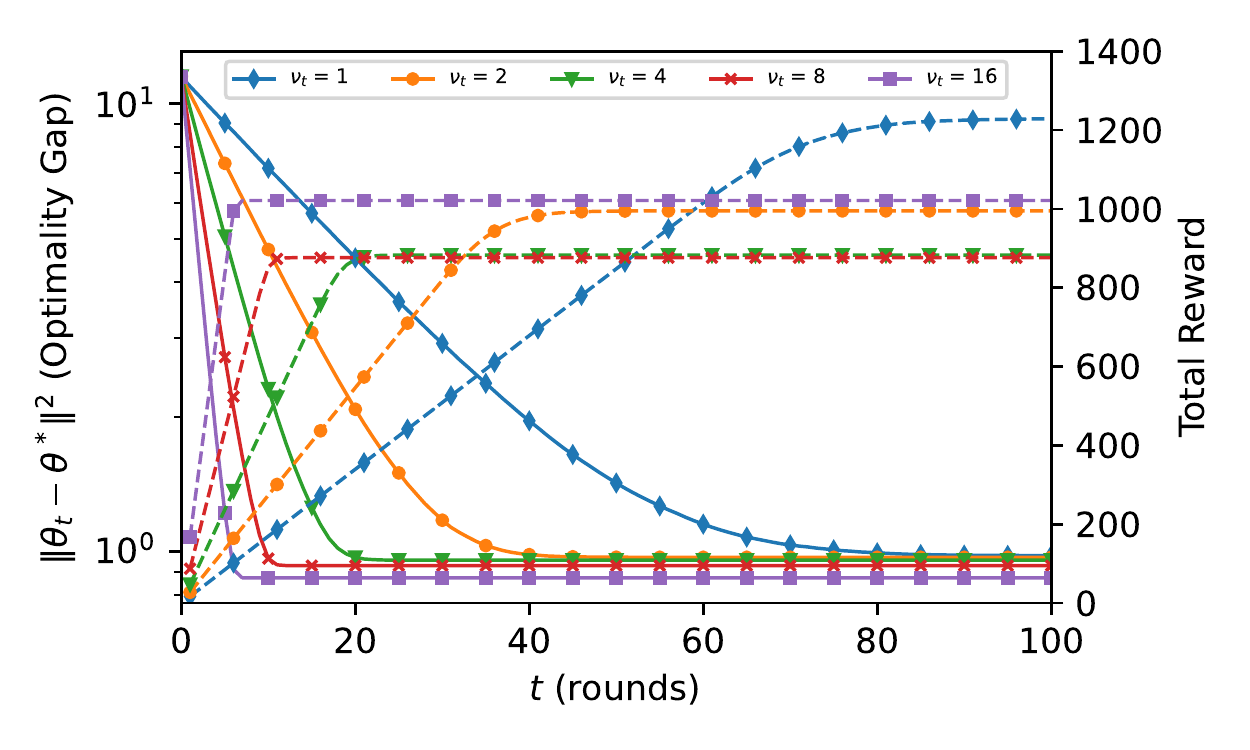}
\vspace{-4mm}
\caption{Linear regression experiments to show the impact of different $\nu_t$ choices on the convergence of the algorithm and the total distributed reward.}
\label{Fig:04}
\end{figure}

\begin{figure}[!t]
\centering
\includegraphics[width=3.4in]{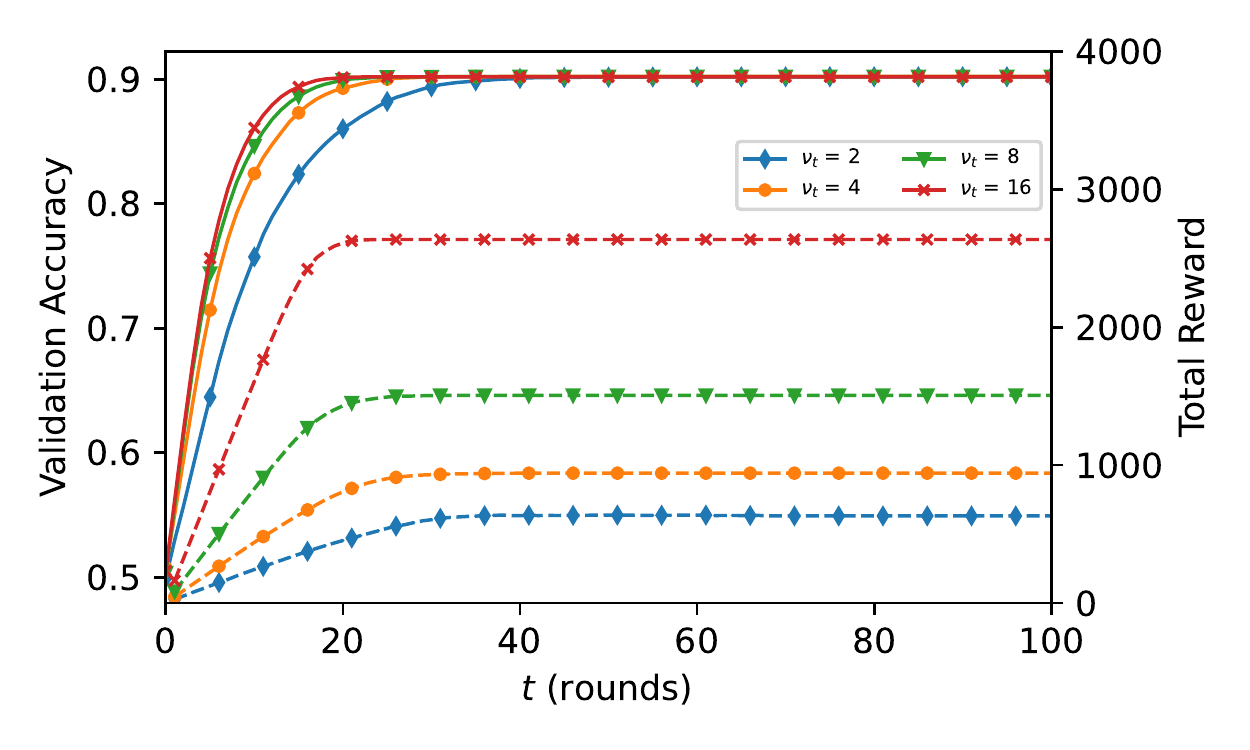}
\vspace{-4mm}
\caption{Logistic regression experiments to show the impact of different $\nu_t$ choices on the convergence of the algorithm and the total distributed reward.}
\label{Fig:05}
\end{figure}

\section{Experiments} \label{Section:Experiments}

In this section, we evaluate the performance of the proposed output agreement based reward mechanism using a synthetically generated dataset and a real-life dataset. In our first set of experiments, we consider a linear regression problem. In each round, agent $k$ collects a mini-batch of $S_{k,t}$ i.i.d. data points, $\zbf_{k,t,i} = (V_{k,t,i}, \ \mathbf{Y}_{k,t,i})$ for $i \in \{ 1, \dots, S_{k,t}\}$, where $\mathbf{Y}_{k,t, i} \in \mathbb{R}^d$ is the feature vector with $\mathbf{Y}_{k,t,i} \stackrel{\mathclap{\normalfont\mbox{i.i.d.}}}{\sim} \mathcal{N}(\boldsymbol{0}, \boldsymbol{\mathbf{I}})$. Each label $V_{k,t,i}$ is generated as $V_{k,t, i} = \mathbf{Y}_{k,t, i}^\top \thbf^* + N_{k,t, i}$ where $\thbf^* \in \mathbb{R}^d$ and $ N_{k,t,i} \stackrel{\mathclap{\normalfont\mbox{i.i.d.}}}{\sim} \mathcal{N}(0, 1)$. A squared error loss is considered, i.e., $\ell (\thbf; Z) = \ell (\thbf; (V, \mathbf{Y})) = \left( \mathbf{Y}^\top \thbf - V \right)^2$. The number of agents is $K=10$ and the dimension is $d=100$. In our second set of experiments, we consider a logistic regression problem trained on the 2020 annual CDC survey data of 400k adults related to their health status \cite{kaggle_dataset}. In both of these problems, the cost function of the agents is $h(S_{k,t}) = S_{k,t} + 1$. Each set of experiments is repeated 250 times. Further details are relegated to appendix.

Fig.~\ref{Fig:03} depicts the reward of agent $k$ for different minibatch size choices, $S_{k,t} \in \{ 2, 4, 8, 16 \}$, when $\nu_t = 8$ and the other agents are cooperative ($S_{i,t} = \nu_t$ for $i \neq k$). The results validate Theorem~\ref{Thm}. In particular, when the agent follows the request of the server ($S_{k,t} \!=\! \nu_t \!= \!8$), his reward is equal to his cost ($R_{k,t}\!=\!9$), which verifies the budget balanced property of the proposed reward mechanism. Choosing other minibatch size values do not increase his utility as his reward is less than his cost. Consequently, it is verified that the proposed reward mechanism has a cooperative equilibrium, since $S_{k,t} = \nu_t$ is the best response minibatch size choice when the other agents follow the server.

Fig.~\ref{Fig:04} depicts the total reward distributed by the server for different requested minibatch size values, $\nu_t \in \{1, 2, 4, 8, 16\}$, given that the agents are cooperative. In this set of experiments, the training stops when the SGD algorithm converges. Afterwards, the total distributed reward remains constant.  The goal is to show the impact of different $\nu_t$ values on the convergence of the SGD algorithm and the total distributed reward. For $\nu_t=1$, the server distributes smaller rewards ($R_{k,t}=2$) within each round; however, the total distributed reward ends up higher than the other $\nu_t$ values because decreasing $\nu_t$ slows down the convergence. Setting $\nu_t = 16$ yields the fastest convergence performance; however, the total distributed reward per round also increases significantly ($R_{k,t}=17$). Considering these two opposing effects, an intermediate value of $\nu_t=4$ or $\nu_t=8$ are appropriate choices for the server. 

For the logistic regression problem, Fig.~\ref{Fig:05} depicts the accuracy of the classifier on the validation data for different $\nu_t$ values. Similar to linear regression problem, increasing $\nu_t$ yields faster convergence. In contrary to the linear regression problem, in this setup, setting $\nu_t = 2$ yields the lowest amount of total distributed rewards at the end of the training. It indicates that further investigation is required to understand the impact of $\nu_t$ choice on the trade-off between total reward and convergence rate.

%% file: 06_Conclusion.tex
\section{Conclusion}
In this work, a federated learning framework is studied where cost-sensitive and strategic agents train a learning model with a server. The main objective of the server is to ensure the reliability of the local model updates of the agents by distributing rewards. The challenge in such a problem is that the gradient updates of the agents is not directly verifiable: The server cannot directly observe or estimate the sample size of the data points that are being used in the computation of the stochastic gradients. To address this challenge, we propose an output agreement based reward mechanism that evaluates each agent's update according to its distances to the gradients submitted by other agents. We show that the proposed reward mechanism has a cooperative Nash equilibrium and satisfies budget balanced property. Our empirical results verify our theoretical analysis and shows that the server can control the convergence performance of the algorithm by varying the reward mechanism parameter $\nu_t$.

The proposed reward mechanism \eqref{Eq:RewardMechanismKnownSigma} requires the knowledge of $\sigma_t^2$, the variance of a single sample stochastic gradient. The server can construct an estimate of this unknown parameter using the gradients sent by the agents as in \eqref{Eq:VarEstDef}. The journal version of our study will include the game theoretical guarantees for this version of the reward mechanism. Furthermore, given an accuracy target, the characterization of an optimal $\nu_t$ sequence which minimizes the total distributed rewards, is still an open problem. 

%% file: App01.tex
\section{Proof of Lemma~\ref{Lemma:OAMech}} \label{Sect:AppLemma2}

\begin{lemma*}
Conditioned on a strategy vector $\mathbf{S}_t$ such that $\mathbf{S}_{-k,t} \neq \boldsymbol{0}$ and $S_{k,t} \neq 0$, it follows that
\begin{equation*}
\E \left[ \| \xbf_{k,t} -\mhat \|^2_2 \ \big | \ \mathbf{S}_t, \thbf_{t-1} \right] = \sigma_t^2 \left( \frac{1}{S_{k,t}} + \frac{1}{|\mathcal{K}_{-k,t}|^2} \sum_{i\in \mathcal{K}_{-k,t}} \frac{1}{S_{i,t}} \right) .
\end{equation*}
\end{lemma*}

\begin{proof}

Let
\begin{align*}
\dot{\xbf}_{k,t} := \xbf_{k,t} - \mbf_t.
\end{align*}
Recall that the stochastic gradients are unbiased noisy copies of $\mbf_t$ and it follows that
\begin{align*}
\E \left[\dot{\xbf}_t | \thbf_{t-1}, S_{k,t} \right] & = \boldsymbol{0} \quad \text{and} \quad \E \left[ \| \dot{\xbf}_{k,t} \|^2 | \thbf_{t-1}, S_{k,t} \right] = \frac{\sigma_t^2}{S_{k,t}}.
\end{align*}
Further, note that $\dot{\xbf}_{k_1,t}$ and $\dot{\xbf}_{k_2,t}$ are uncorrelated, for $k_1 \neq k_2$, conditioned on $\thbf_{t-1}$:
\begin{align*}
\E \left[\dot{\xbf}_{k_1,t}^\top \dot{\xbf}_{k_2,t} \big | \ \thbf_{t-1}, S_{k_1,t}, S_{k_2,t} \right] & = \frac{1}{S_{k_1,t}} \frac{1}{S_{k_2,t}} \sum_{i=1}^{S_{k_1, t}} \sum_{j=1}^{S_{k_2, t}} \left( \E \left[ \nabla_{\thbf} \ \ell \big( \thbf_{t-1}; \zbf_{k_1,t,i} \big) - \mbf_t \right] \right)^\top \E \left[ \nabla_{\thbf} \ \ell \big( \thbf_{t-1}; \zbf_{k_2,t,i} \big) - \mbf_t \right] \\
& = 0.
\end{align*}
Then, 
\begin{align*}
\E \left[ \| \xbf_{k,t} -\mhat \|^2 \ \big | \ \mathbf{S}_t, \thbf_{t-1} \right] & = \E \left[ \left \| \xbf_{k,t} - \frac{1}{|\mathcal{K}_{-k,t}|} \sum\nolimits_{i \in \mathcal{K}_{-k,t}} \xbf_{i,t} \right \|^2 \ \bigg | \ \mathbf{S}_t, \thbf_{t-1}  \right] \\
& =  \E \left[ \left \| \dot{\xbf}_{k,t} - \frac{1}{|\mathcal{K}_{-k,t}|} \sum\nolimits_{i \in \mathcal{K}_{-k,t}} \dot{\xbf}_{i,t} \right \|^2 \ \bigg | \ \mathbf{S}_t, \thbf_{t-1}  \right] \\
& = \E \left[ \| \dot{\xbf}_{k,t} \|^2 \ \big| \ \thbf_{t-1}, S_{k,t} \right] +  \frac{1}{|\mathcal{K}_{-k,t}|^2} \sum\nolimits_{i \in \mathcal{K}_{-k,t}} \E \left[ \| \dot{\xbf}_{k,t} \|^2 \ \big| \ \thbf_{t-1}, S_{i,t} \right] \\
& = \frac{\sigma_t^2}{S_{k,t}} + \frac{1}{|\mathcal{K}_{-k,t}|^2} \sum\nolimits_{i \in \mathcal{K}_{-k,t}} \frac{\sigma_t^2}{S_{i,t}}.
\end{align*}

\end{proof}

%% file: App02.tex
\section{Proof of Theorem~\ref{Thm}}

\begin{theorem*}
The output agreement based reward mechanism \eqref{Eq:RewardMechanismKnownSigma} has a cooperative Nash equilibrium and it is budget balanced.
\end{theorem*}
\begin{proof}
Under the proposed output agreement based reward mechanism \eqref{Eq:RewardMechanismKnownSigma}, the expected utility of an agent is given by
\begin{align*}
u_{k,t} (\mathbf{S}_t) & = \E \left[ R_{k,t} \ |\ \thbf_{t-1}, \mathbf{S}_t \right] - h(S_{k,t}) \\
& = \E \left[ h(\nu_t) + h'(\nu_t) \left( \nu_t  \frac{K}{K-1} - \frac{\nu_t^2}{\sigma_t^2} \| \xbf_{k,t} - \mhat \|^2 \right) \ \bigg | \ \thbf_{t-1}, \mathbf{S}_t \right] - h(S_{k,t}) \\
& = h(\nu_t) +  h'(\nu_t) \left( \nu_t  \frac{K}{K-1} - \nu_t^2 \left( \frac{1}{S_{k,t}} + \frac{1}{|\mathcal{K}_{-k,t}|^2} \sum\nolimits_{i \in \mathcal{K}_{-k,t}} \frac{1}{S_{i,t}} \right) \right) - h(S_{k,t}),
\end{align*}
where the last equality follows from Lemma~\ref{Lemma:OAMech}. Let $\boldsymbol{\nu}_{-k,t} := [\nu_t \dots \nu_t]_{K-1 \times 1}$. It follows that
\begin{align*}
u_{k,t} \left( S_{k,t}, \mathbf{S}_{-k,t} = \boldsymbol{\nu}_{-k,t} \right) & = h(\nu_t) +  h'(\nu_t) \left( \nu_t  \frac{K}{K-1} - \nu_t^2 \left( \frac{1}{S_{k,t}} + \frac{1}{(K-1)^2} \sum_{i =1, i\neq k}^K \frac{1}{\nu_t} \right) \right) - h(S_{k,t}) \\
& = h(\nu_t) + \nu_t h'(\nu_t) \left( 1 - \frac{\nu_t}{S_{k,t}} \right) - h(S_{k,t}).
\end{align*}
For $s > 0$, an auxiliary function $J(s)$ is defined as:
\begin{align*}
J (s) = h(\nu_t) + \nu_t h'(\nu_t) \left( 1 - \frac{\nu_t}{s} \right) - h(s).
\end{align*}
According to Assumption~\ref{Assumption:CostFunction}, for $s > 0$, $h$ is twice differentiable and $h''(s) > 0$. Thus,
\begin{align*}
J'(s) = \nu_t^2 h'(\nu_t) \frac{1}{s^2} - h'(s) \quad \text{and} \quad
J''(s) = - 2\nu_t^2 h'(\nu_t) \frac{1}{s^3} - h''(s) < 0. 
\end{align*}
It follows that $s = \nu_t$ is the unique maximizer of $J(s)$. Thus, $S_{k,t} = \nu_t$ is the unique maximizer of $u_{k,t} \left( S_{k,t}, \mathbf{S}_{-k,t} = \boldsymbol{\nu}_{-k,t} \right)$ for $S_{k,t} \in \{1, \dots, n\}$. Note that 
\begin{align*}
u_{k,t} (\mathbf{S}_t = \boldsymbol{\nu}_t) = u_{k,t} (S_{k,t} = 0, \mathbf{S}_{-k,t} = \boldsymbol{\nu}_{-k,t}) = 0. 
\end{align*}
Consequently, 
\begin{align*}
u_{k,t} (\mathbf{S}_t = \boldsymbol{\nu}_t) \geq u_{k,t} (S_{k,t} = s_{k,t}, \mathbf{S}_{-k,t} = \boldsymbol{\nu}_{-k,t})
\end{align*}
for all $s_{k,t} \in \{0, 1, \dots, n\}$ and $\mathbf{S}_t = \boldsymbol{\nu}_t$ is a Nash equilibrium. Since $u_{k,t} (\mathbf{S}_t = \boldsymbol{\nu}_t) = 0$, the output agreement based reward mechanism \eqref{Eq:RewardMechanismKnownSigma} satisfies the budget balanced property \eqref{Eq:IRDef}.
\end{proof}

%% file: main.bbl
\begin{thebibliography}{10}
\providecommand{\url}[1]{#1}
\csname url@samestyle\endcsname
\providecommand{\newblock}{\relax}
\providecommand{\bibinfo}[2]{#2}
\providecommand{\BIBentrySTDinterwordspacing}{\spaceskip=0pt\relax}
\providecommand{\BIBentryALTinterwordstretchfactor}{4}
\providecommand{\BIBentryALTinterwordspacing}{\spaceskip=\fontdimen2\font plus
\BIBentryALTinterwordstretchfactor\fontdimen3\font minus
  \fontdimen4\font\relax}
\providecommand{\BIBforeignlanguage}[2]{{%
\expandafter\ifx\csname l@#1\endcsname\relax
\typeout{** WARNING: IEEEtran.bst: No hyphenation pattern has been}%
\typeout{** loaded for the language `#1'. Using the pattern for}%
\typeout{** the default language instead.}%
\else
\language=\csname l@#1\endcsname
\fi
#2}}
\providecommand{\BIBdecl}{\relax}
\BIBdecl

\bibitem{Konecny45630}
J.~Konečný, H.~B. McMahan, D.~Ramage, and P.~Richtarik, ``Federated
  optimization: Distributed machine learning for on-device intelligence,''
  \emph{arXiv:1610.02527 [cs.LG]}, Oct. 2016.

\bibitem{McMahan2018learning}
H.~B. McMahan, D.~Ramage, K.~Talwar, and L.~Zhang, ``Learning differentially
  private recurrent language models,'' in \emph{Int. Conf. Learning
  Representations (ICLR’18)}, 2018.

\bibitem{FengNiyato19}
S.~Feng, D.~Niyato, P.~Wang, D.~I. Kim, and Y.-C. Liang, ``Joint service
  pricing and cooperative relay communication for federated learning,'' in
  \emph{Int. Conf. on Internet of Things (iThings) and IEEE Green Comput. and
  Commun. (GreenCom) and IEEE Cyber, Physical and Social Comput. (CPSCom) and
  IEEE Smart Data (SmartData)}, 2019, pp. 815--820.

\bibitem{PandeyTran20}
S.~R. Pandey, N.~H. Tran, M.~Bennis, Y.~K. Tun, A.~Manzoor, and C.~S. Hong, ``A
  crowdsourcing framework for on-device federated learning,'' \emph{IEEE Trans.
  on Wireless Commun.}, vol.~19, no.~5, pp. 3241--3256, 2020.

\bibitem{KangXiong19}
J.~Kang, Z.~Xiong, D.~Niyato, H.~Yu, Y.-C. Liang, and D.~I. Kim, ``Incentive
  design for efficient federated learning in mobile networks: A contract theory
  approach,'' in \emph{IEEE VTS Asia Pacific Wireless Commun. Symp.
  (APWCS'19)}, 2019, pp. 1--5.

\bibitem{DingFang21}
N.~Ding, Z.~Fang, and J.~Huang, ``Optimal contract design for efficient
  federated learning with multi-dimensional private information,'' \emph{IEEE
  J. on Selected Areas in Commun.}, vol.~39, no.~1, pp. 186--200, 2021.

\bibitem{Tang21}
M.~Tang and V.~W. Wong, ``An incentive mechanism for cross-silo federated
  learning: A public goods perspective,'' in \emph{IEEE Conf. on Computer
  Commun. (INFOCOM'21)}, 2021, pp. 1--10.

\bibitem{DongZhang20}
L.~Dong and Y.~Zhang, ``Federated learning service market: A game theoretic
  analysis,'' in \emph{Int. Conf. on Wireless Commun. and Signal Process.
  (WCSP'20)}, 2020, pp. 227--232.

\bibitem{DengLyu21}
Y.~Deng, F.~Lyu, J.~Ren, Y.-C. Chen, P.~Yang, Y.~Zhou, and Y.~Zhang, ``Fair:
  Quality-aware federated learning with precise user incentive and model
  aggregation,'' in \emph{IEEE Conf. on Computer Commun. (INFOCOM'21)}, 2021,
  pp. 1--10.

\bibitem{Tahanian21}
E.~Tahanian, M.~Amouei, H.~Fateh, and M.~Rezvani, ``A game-theoretic approach
  for robust federated learning,'' \emph{Int. J. of Engineering}, vol.~34,
  no.~4, pp. 832--842, 2021.

\bibitem{FLEconSurvey22}
X.~Tu, K.~Zhu, N.~C. Luong, D.~Niyato, Y.~Zhang, and J.~Li, ``Incentive
  mechanisms for federated learning: From economic and game theoretic
  perspective,'' \emph{IEEE Trans. on Cognitive Commun. and Networking},
  vol.~8, no.~3, pp. 1566--1593, 2022.

\bibitem{Akbay22}
A.~B. Akbay and J.~Zhang, ``Distributed learning with strategic users: A
  repeated game approach,'' \emph{Proc. of the AAAI Conf. on Artificial
  Intelligence (AAAI'22)}, vol.~36, no.~6, pp. 5976--5983, Jun. 2022.

\bibitem{PressDyson12}
W.~H. Press and F.~J. Dyson, ``Iterated prisoner{\textquoteright}s dilemma
  contains strategies that dominate any evolutionary opponent,'' \emph{Proc.
  Natl. Acad. Sci}, vol. 109, no.~26, pp. 10\,409--10\,413, 2012.

\bibitem{Dean12}
J.~Dean, G.~Corrado, R.~Monga, K.~Chen, M.~Devin, M.~Mao, M.~Ranzato,
  A.~Senior, P.~Tucker, K.~Yang, Q.~Le, and A.~Ng, ``{Large Scale Distributed
  Deep Networks},'' in \emph{Advances in Neural Inform. Proc. Systems
  (NIPS'12)}, vol.~25, 2012.

\bibitem{Bottou18}
L.~Bottou, F.~E. Curtis, and J.~Nocedal, ``Optimization methods for large-scale
  machine learning,'' \emph{SIAM Review}, vol.~60, no.~2, pp. 223--311, 2018.

\bibitem{vonAhn04}
L.~von Ahn and L.~Dabbish, ``Labeling images with a computer game,'' in
  \emph{Proc. of the SIGCHI Conf. on Human Factors in Comput. Systems
  (CHI'04)}, 2004, p. 319–326.

\bibitem{kaggle_dataset}
\BIBentryALTinterwordspacing
K.~Pytlak, ``Personal key indicators of heart disease,'' Accessed Nov. 18,
  2022. [Online]. Available:
  \url{https://www.kaggle.com/datasets/kamilpytlak/personal-key-indicators-of-heart-disease}
\BIBentrySTDinterwordspacing

\end{thebibliography}
